\documentclass[runningheads]{llncs}
\usepackage[cmintegrals]{newtxmath}
\usepackage{cite}
\usepackage{amsmath,amssymb,amsfonts}
\usepackage{multicol}
\usepackage{makecell}
\usepackage{graphicx}
\usepackage{textcomp}
\usepackage{xcolor}
\usepackage{textcomp}
\usepackage[ruled,vlined,linesnumbered]{algorithm2e}
\usepackage{soul}
\usepackage{algpseudocode}
\usepackage{subcaption}
\begin{document}

\title{Discovering Frequent Gradual Itemsets with Imprecise Data}
\titlerunning{Discovering Gradual Itemsets with Imprecise Data}

\author{Micha\"el Chirmeni Boujike\inst{1}
\and
Jerry Lonlac\inst{2, 5}
\and
Norbert Tsopze\inst{1,4}
\and Engelbert Mephu Nguifo\inst{3}
}
\authorrunning{M. Chirmeni et al.}

\institute{Department of computer science - University of Yaound\'{e} I, Cameroon\\
\email{\{chir.michael47,tsopze.norbert\}@gmail.com}
\and
Univ. Lille, Research Center, IMT Lille Douai, 59500 Douai, France
\email{jerry.lonlac@imt-lille-douai.fr}
\and
University Clermont Auvergne, LIMOS CNRS UMR 6158, F-63178 Aubi\`ere
\email{engelbert.mephu\_nguifo@uca.fr}\\
\and
Sorbonne University, IRD, UMMISCO, F-93143, Bondy, France, \\
\and
Department of Computer Engineering - ENSET - University of Douala, Cameroon \\
}
\maketitle              
\begin{abstract}
The gradual patterns that model the complex co-variations of attributes of the form "\textit{The more/less X, The more/less Y}" play a crucial role in many real world applications where the amount of numerical data to manage is important, this is the biological data, or medical data. Recently, these types of patterns have caught the attention of the data mining community, where several methods have been defined to automatically extract and manage these patterns from different data models. However, these methods are often faced the problem  of managing the quantity of mined patterns, and in many practical applications, the calculation of all these patterns can prove to be intractable for the user-defined frequency threshold and the lack of focus leads to generating huge collections of patterns. 
Moreover another problem with the traditional approaches is that the concept of gradualness is defined just as an increase or a decrease. Indeed, a gradualness is considered as soon as the values of the attribute on both objects are different.
As a result, numerous quantities of patterns extracted by  traditional algorithms  can be presented to the user although their gradualness is only a noise effect in the data.
To address this issue, this paper suggests to introduce the gradualness thresholds from which to consider an increase or a decrease.  
In contrast to literature approaches, the proposed approach takes into account the distribution of attribute values, as well as the user's preferences on the gradualness threshold. 
The proposed algorithm makes it possible to extract gradual patterns on certain databases where state-of-the art gradual patterns mining algorithms fail due to too large search space.
Moreover, results from an experimental evaluation on real databases show that the proposed algorithm is scalable, efficient, and can eliminate numerous patterns that do not verify specific gradualness requirements to show a small set of patterns to the user.


\end{abstract}

\section{Introduction}
\label{sec:intro} 
Human reasoning is most often based on inaccurate or incomplete data; indeed, it is easy for a human being to determine if a person is short or tall without knowing its exact size; which is not the case for a computer because it processes the exact data. Transmit the faculties of human reasoning to a computer was initiated by \cite{zadeh1996fuzzy} whose purpose was to have imprecise data processed on a computer. Work based on fuzzy logic \cite{zadeh1996fuzzy} and classical logic \cite{duneja2012survey} has been carried out for the extraction of knowledge in categorical databases but has many difficulties in numerical databases.

Most works about correlations extraction \cite{vizhi2012data, jagtap2012role} on the processing of numerical databases proceed by discretizing numerical attributes, thus returning to the case of categorical data processing. But as numerical databases become important, finer dig methods are needed to extract more expressive knowledge representing the frequent variability between numerical values. To this end, the gradual patterns that model the complex co-variations of attributes of the form "\textit{the more/less X, the more/less Y"} were proposed.

Thanks to the knowledge that they provide in many applications \cite{lonlac2018approach, Di-JorioLT09}, gradual patterns have gained popularity, and several effective algorithms have been proposed to extract these patterns from numerical databases. These algorithms, which differ for the most part from the gradual semantics \cite{lonlac2018approach, Di-JorioLT09, Di-JorioLT08, laurent2010pgp}, very often come up against a very large number of extracted patterns, making their exploitation difficult by the user. Thus, some reported works propose to use constraints to prune the search space during mining process and focus on the interesting patterns. This is the case of the approach proposed by \cite{lonlac2018approach} which takes into account the temporal constraint between objects during the mining process. Such an approach is adapted to the context with a temporal order among objects and makes it possible to generate only patterns that are useful and relevant to the user \cite{lonlac2018approach}. Several other approaches like \cite{AyouniLYP10} proposed to extract closed gradual patterns which are a compact representation of the entire gradual patterns.
Although the proposed algorithms allow to significantly reduce the number of generated patterns, this patterns number remains very high in some applications. Indeed, the distribution of the data is not always uniform on all the attributes, as the gradualness is considered only in terms of increase and decrease of the attribute values, the number of extracted patterns can be more reduced and more significant for interpretation, if we integrate in the mining process a threshold of increase or decrease from which a gradualness should be considered. For example, in Medicine, the normal body temperature is around 37\textdegree C and professionals consider as fever sign the temperature greater than 38\textdegree C. This means that they can not make any decision based on temperature variation between 37\textdegree C and 38\textdegree C. 

To deal with this problem, we propose a new algorithm called \textit{GRAPGT} (GRAdual Patterns with Gradualness Threshold) to  automatically extract gradual patterns, by integrating a variation threshold from which to consider an attribute gradualness (increase/decrease). This allows to take into account the user preferences on each attribute of the database during the mining process.

Embedded in the gradual semantic proposed by \cite{lonlac2018approach} or ParaMiner \cite{NegrevergneTRM14} or in GRITE \cite{Di-JorioLT09}, this approach makes it possible to reduce the number of generated patterns. 

The main contributions of this paper is fourfolds:
\begin{itemize}
    \item We propose how to take into account user preferences in the mining process (in terms of variation threshold) on the variations of attribute values. 
    \item We propose to reduce the number of generated patterns through the introduction of a variation threshold from which to consider a gradualness (increase/decrease). The proposed approach let the user more precisely specify requirements about the gradualness threshold of patterns to be discovered.
    \item We study the impact of the gradualness threshold on the gradual itemsets mimings process and on the patterns to be discovered.
\item We conduct many experiments on three real world datasets and one synthetic dataset in order to validate this proposition and to compare its results to that of the original algorithms in terms of scalability, runtime, and memory consumption.
\end{itemize}

The rest of this article is organized as follows: first of all, we present some preliminary notions on the  gradual patterns mining in Section \ref{sec:prelim_relatedworks}. 
In Section \ref{sec:approach}, we then describe our approach to extract gradual patterns with constraints on the variations of attribute values. Before concluding, we present and discuss experimental results in Section \ref{sec:expe}.


\section{Gradual itemsets}
\label{sec:prelim_relatedworks}

In this section, we provide some useful definitions and formally describe the problem of mining frequent gradual itemsets (patterns) by defining the support of such itemsets in data. We also present some of the state-of-the-art approaches
that automatically extract such patterns.

\subsection{Preliminary definitions}
\label{sec:prelim}
The problem of mining gradual itemsets consists in mining attribute co-variations in a numerical data set of the form \textit{"The more/less X, . . . , the more/less Y'"}. 
We assume herein that we are given a database $\Delta$ containing a set of objects $\mathcal{T}$ that defines a relation on an attribute set $\mathcal{I}$ with numerical values.
Let $t.i$ denote the value of  attribute $i$ over object $t$ for all $t\in \mathcal{T}$.

To illustrate the notion of gradual itemsets, we consider the database $\Delta$ given by Table \ref{exampledataset1}. This database containing a set of objects  defining a relation on an attribute set with numerical values. It describes the blood sugar level measured in a patient with diabete for $8$ days ($t_1, \ldots, t_8$). The columns represent this rate in grams per liter (g/l) taken on an empty stomach ($a_1$), after lunch ($a_2$), after dinner ($a_3$) and at bedtime ($a_4$) while the lines represent the days of observation.


\begin{table}[!h]
		\begin{center}
		\renewcommand{\arraystretch}{1.8}
      \begin{tabular}{lllll}
                \hline
      \textbf{Tid}~ & ~~\textbf{$a_1$}~~ & ~~\textbf{$a_2$}~~  & ~~\textbf{$a_3$}~~  & ~~\textbf{$a_4$}  \\
                \hline
                $t_1$ & 1.18 & 2.36 & 2.36 & 1.58 \\
                \hline
                $t_2$ & 1.32 & 3.01 & 2.58 & 2.45 \\
                \hline
                $t_3$ & 1.25 & 2.56 & 2.31 & 2.25 \\
                \hline
                $t_4$ & 1.30 & 3.15 & 2.80 & 2.36 \\
                \hline
                 $t_5$ & 1.04 & 2.75 & 2.30 & 2.35 \\
                \hline
                $t_6$ & 1.48 & 3.56 & 2.75 &  2.53 \\
                \hline
                $t_7$ & 1.65 & 3.70 &  2.60 & 2.40  \\
                \hline
                $t_8$ & 1.28 & 4.08  & 3.09 &  2.90 \\
                \hline
            \end{tabular}
      \caption{Example of numerical database ($\Delta$).}
        \label{exampledataset1}
        \end{center}
\end{table}

Each attribute will hereafter be considered twice: once to indicate its increase, and another to indicate its decrease, using the $\leq$ and $\geq$ operators. This leads to new kinds of items, called gradual items. 
\begin{definition}[Gradual item]
\label{def:gradualItem}
Let $\Delta$ be a data set defined on a numerical attribute set $\mathcal{I}$. A gradual item is defined under the form $i^*$, where $i$ is an attribute of  $\mathcal{I}$ and $* \in \{\leq, \geq\}$. 
\end{definition}

A gradual item $i^*$ express a variation of an attribute $i$ associated to one of the values \textit{- increase or decrease-}.

If we consider the numerical database of Table \ref{exampledataset1}, $a_{1}^{\geq}$ (respectively $a_{1}^{\leq}$) is a gradual item meaning that the values of attribute \textit{$a_1$} are increasing (respectively decreasing). These two gradualness respectively increasing and decreasing are respected by the sequence of objects $\langle t_5, t_1, t_3, t_8,t_4, t_2, t_6, t_7 \rangle$ and $\langle t_7,  t_6, t_2, t_4,  t_8, t_3, t_1, t_5, \rangle$.

A gradual pattern expresses monotonous variations of values of several attributes and is defined as follows:

\begin{definition}[Gradual pattern]
\label{def:gradualItemset}
A gradual pattern $G = (i_1^{*_1}, ... , i_k^{*_k})$ is a non-empty set of gradual items. A $k$-itemset is an itemset containing $k$ $(k > 1)$ gradual items.
\end{definition}

In our example (see Table \ref{dataset}), $G_1 = (a_{1}^{\geq}, a_{2}^{\geq})$  is a gradual itemset means "the more the blood sugar level taken on an empty stomach, the more the blood sugar level taken after lunch". This gradual itemset is satisfied by the sequence of objects $\langle t_1, t_3, t_2, t_6, t_7 \rangle$ in Table \ref{dataset}.

The \textbf{support} (frequency) of a gradual itemset amounts to the extent to which a gradual pattern is present in a given database. Several support definitions have been proposed in the literature (e.g., \cite{Hullermeier02,CaldersGJ06,BerzalCSVS07,Di-JorioLT08,Di-JorioLT09,lonlac2018approach}), showing that gradual patterns can follow different semantics. 
We briefly describe in the next section these ways of defining the support of gradual itemsets in the data.
In order to define these supports, we introduce the following definitions:

\begin{definition}[Gradual Pattern Extension]
\label{def:ext}
Let $G = (i_1^{*_1}, ... , i_k^{*_k})$ be a gradual itemset and $s = \langle t_1, t_2, \ldots, t_n \rangle$ a sequence of objects. $s$ is an  extension of $G$ if $\forall p, 1 \le p \le k$, $\forall j, 1 \le j < n$, 
$t_j.i_p *_p t_{j+1}.i_p$ holds.
\end{definition}

From Table \ref{exampledataset1}, the sequence of objects $\langle t_1, t_3, t_2, t_6, t_7 \rangle$ is an extension of gradual itemset $G_1 = (a_{1}^{\geq}, a_{2}^{\geq})$. 
For a gradual itemset there might be several extensions. $G_1$ is also verified by the sequence $\langle t_5, t_2, t_6, t_7 \rangle$. In general, a gradual itemset is relevant to explain the gradualness occurring in these extensions.
The computation of the support value of a gradual pattern in a given database $\Delta$ amounts to the extend to which the gradual pattern is present in $\Delta$. This is assessed by considering the most representative extension, i.e., the extension having the largest size.

\begin{definition}[Gradual Itemset Support]
Let $\Delta$ be a numerical database and $G$ be a gradual itemset of $\Delta$ ($|\Delta|$ is its rows number). Let $\Delta_G = \{ s_1, \ldots, s_m\}$ be the set of all the longest sequence of objects respecting $G$
Then, $support(G, \Delta) = \frac{max \{|s_i|,~ 1 \leq i \leq m,~ s_i \in ~\Delta_G \}}{|\Delta|}$.
\end{definition}

Considering Table \ref{exampledataset1}, the longest sequence of objects that respects the gradual itemset $G_1$ is $\langle t_1, t_3, t_2, t_6, t_7 \rangle$ of size $5$. Then, $support(G_1, \Delta) = \frac{5}{8}$, meaning that five among the eight input objects can be ordered consecutively according to $G_1$.
A gradual itemset is frequent if its support is greater than or equal to a minimum threshold set by the user.

\begin{definition}[Complementary Gradual Itemset]
\label{def:patternCompl}
Let $G = (i_1^{*_1}, ..., i_k^{*_k})$ be a gradual itemset, and $c$ be a function such that "$c(\geq) = '\leq'$ and $c(\leq) = '\geq'$". Then $c(G) = (i_1^{c(*_1)}, ..., i_k^{c(*_k}) $  is the complementary (symmetric) gradual itemset of $G$.
\end{definition}

Any gradual pattern admits a complementary gradual pattern where the items are the same and the variations are all reversed. The complementary  gradual pattern of $(a_{1}^{\geq}, a_{2}^{\geq})$ is $(a_{1}^{\leq}, a_{2}^{\leq})$.

\begin{definition}[Frequent Gradual Patterns Mining Problem]
Let $\Delta$ be a numerical database and $minSupp$ a minimum  support threshold.
The problem of mining gradual patterns is to find the complete set of frequent gradual patterns of $\Delta$ with respect to $minSupp$ i.e. finding the set $\{G|~support(G, \Delta)\geq minSupp \}$.
\end{definition}

\subsection{Closed Gradual Patterns}
\label{closedGradual}
The notion  of the closure of the itemset is very important in the pattern mining task as it allows to obtain concise representation of patterns without loss of information. This notion has been widely studied in classical itemset mining framework \cite{NguifoN98, PasquierBTL99, YahiaHN06, YahiaGN09}.
It was introduced for the first time when extracting gradual patterns in \cite{AyouniLYP10}, where the authors propose a pair of functions $(f, g)$ defining a closure operator \cite{GanterW99, Mephu-Nguifo94} for the gradual patterns.
Given a list of sequences of objects  $\Delta_G$ from a database, $f$ returns the gradual itemset $G$ respecting all objects sequence in $\Delta_G$, while the function $g$ returns the set of maximal sequences of objects $Delta_G$ which respects the variations of all gradual items in $G$.
A gradual itemset $G$ is said to be closed if we have $f(g(G)) = G$.


\section{Related works}
\label{sec:relatedworks}

The problem of mining gradual itemsets consists of discovering frequent simultaneous attribute co-variations in numerical databases.
Several works have been devoted to the gradual itemset mining problem and many algorithms using different frameworks and formalisations have been developed to automatically extract these itemsets from numerical database.
These itemsets have been extensively used for fuzzy command (e.g., “the \textit{closer} the wall, the harder the break”). In this context, we are only interested in the knowledge revealed by such itemsets, and the focus is not on how to mine these gradual itemsets from huge databases and moreover these itemsets are not automatically extracted , they are provided by an expert, which is not always realistic in practice.
In recent years, gradual itemsets has received much attention from the data mining community and several methods have been defined for automatically extracting these itemsets from numerical data. 

Most of the proposed methods for extracting gradual itemsets generally differ in their application on the basis of the type of data (temporal data, temporal data sequences, data stream, multi-relational Data, graph data, etc.) from which the pattern extraction is performed.

In the follow,  we provide an overview of the methods used by some of the most efficient gradual itemset extraction algorithms according to the data model considered and underline their advantages and limitations.

\subsection{On the gradual itemsets extraction from numerical data}

\subsubsection{Regression-based approach:}

 \cite{Hullermeier02} proposes the first interpretation of gradual dependency as a co-variation constraint and models gradual itemsets using statistical linear regression. So, linear regression is performed between the attribute pairs and the validity of the graduality is relied on the normalised mean squared error $R^2$, together with the slope of the regression line. The gradual tendencies are revealed for attribute pairs that are strongly correlated with a strong slope of the regression line.
 
 \subsubsection{Approach based on the discovery of association rules:}
\cite{BerzalCSVS07} was the first using data mining methods through an adaptation of the \textit{Apriori} algorithm to extract the gradual itemsets. The authors evaluate the support of a gradual itemset by considering the proportion of objects couples that verify the constraints expressed by all the gradual items in the itemset. In that work, the reported gradual tendencies are based on correlation between attributes. The authors consider correlation in terms of the rankings induced by the attributes and not in terms of their values. Due to memory complexity, the algorithm reported in \cite{BerzalCSVS07} is limited to the extraction of gradual itemsets with maximum length equals to $3$.

\subsubsection{Conflict sets based approach:}
In the same way \cite{Di-JorioLT08} proposes a gradual pattern extraction approach based on a heuristic for computing an approximate support value of a gradual pattern.
For a given gradual itemset, this approach removes the objects (conflicts set) which prevent the maximum number of objects in the database to be ordered according to the gradual itemset.
In this approach, the author constructs the candidates by ordering the transactions (objects) according to the values of the attributes. For the generation of a candidate gradual pattern of size $2$ for example, the author will order the transactions in a direction of graduality (increasing or decreasing) of the first item then, will associate the second item and check if the order of the transactions is respected according to the values of the item. In the case where the order is not respected, the author proceeds by a deletion of the transactions (set of conflict) preventing this order. The support of a gradual pattern is considered here as the cardinality of the longest ordered list of transactions respecting the gradual itemset.
The major limitation of this approach is that it is not complete. Indeed, the use of a heuristic implies that in some cases the frequency of a gradual pattern may be undervalued. 
This phenomenon is explained by the fact that the author makes choices each time a set of conflicts is encountered. However, whatever the policy adopted for the choice to be made, it may have a loss of gradual reasons, this being explained by the fact that when generating a gradual itemset of size $n$, it is impossible to predict the best choice that will have an impact on the generation of gradual itemsets of size $n+1$. One solution is to keep all possible choices.

\subsubsection{Precedence graph based approach:}
To deal with the disadvantage revealed in approach  \cite{Di-JorioLT08}, \cite{Di-JorioLT09} proposes a more complete approach named GRITE (GRadual ITemset Extraction) to automatically extract gradual itemsets from large databases. The authors consider the same definition of the support of a gradual pattern proposed in \cite{Di-JorioLT08} and proposes a new method based on precedence graphs. In this method, the data are represented by a graph whose nodes are defined as the objects of the transactional database and the links stand for the precedence relation derived from the attributes taken into account. The author adopts a binary representation of the graph by a matrix. The support of the gradual pattern considered is defined as the length of the longest path in the graph. This approach makes it possible to efficiently generate itemsets of size $n+1$ from gradual itemsets of size $n$. 

In \cite{LaurentLR09}, the authors propose an algorithm that combines the principles of several existing approaches and benefits from efficient computational properties to extract frequent gradual itemsets. In fact, they consider the formulation proposed in \cite{BerzalCSVS07}, and propose an algorithm that take into account the binary structure used in \cite{Di-JorioLT09}.  To evaluate the support of a gradual itemset, the authors use the Kendall tau ranking correlation coefficient that computes the number of object pairs which are consistent or inconsistent in the database, to be in agreement with the considered gradual itemset.

\subsection{On the reduction of the number of gradual itemsets extracted}

Most of these methods come up against the problem of managing the very high quantity of extracted patterns. In practice, the number of frequent gradual itemsets can be large, making their interpretation by the expert almost impossible. One solution to reduce the number of extracted itemsets is to use the constraints in the process mining to focus on interest patterns.

From a set of specific gradual itemsets, as the closed gradual patterns \cite{AyouniLYP10}, it is possible to regenerate the set of all gradual itemsets. Moreover, with the closed itemsets redundant information is avoided. 
Following this idea, \cite{DoTLNTA15} proposes an algorithm named \textit{GLCM}, based on an extension of LCM \cite{UnoKA04} algorithm principle. \textit{GLCM} permits to efficiently compute gradual itemsets over large real-world databases with a time complexity linear in the number of closed frequent gradual itemsets and a memory complexity constant w.r.t. the number of closed frequent gradual itemsets. The \textit{GLCM} algorithm exploits the binary structure proposed in \cite{Di-JorioLT09} to compute the support and the closure of gradual itemsets.
Indeed, the proposed approach in \cite{AyouniLYP10} allows to reduce the number of extracted patterns as a post-processing step which is not efficient. This approach is just a post-processing of \cite{Di-JorioLT09}.  It does not allow to benefit from the runtime and memory reduction and thus does not provide any added value for running the algorithms. The authors of \cite{DoTLNTA15} cope with this by proposing an algorithm that reduces the number of patterns during the mining process.


In \cite{NegrevergneTRM14}, \texttt{ParaMiner}, a generic and parallel algorithm for closed pattern mining, is proposed.
It is based on the principle of pattern enumeration in strongly accessible set systems and its efficiency is due to a dataset reduction technique called \textit{EL-reduction}, combined with a technique for performing dataset reduction in a parallel execution on a multi-core architecture.

\subsection{On the gradual itemsets extraction from the complex data}

Most of these algorithms use data mining techniques to extract gradual patterns. However, they are not relevant for extracting gradual patterns in certain application domains where numerical data present particular forms (e.g., temporal, stream, relational, or noisy data).
So, some recent  works have instead focused on extracting  variants of gradual patterns on the numerical data supplied with specific constraints for expressing another kind of knowledge.

\subsubsection{Extraction gradual itemsets from stream data:}

In \cite{NinLP10}, an approach based on B-Trees and OWA (Ordered Weighted Aggregation) operator \cite{Yager93,Yager88} is proposed to mine data streams for gradual patterns.
 \cite{PhanIMPT15} proposes the relational gradual pattern
concept, which enables to examine the correlations between attributes from a graduality point of view in multi-relational data. 

\subsubsection{Extraction gradual itemsets from noisy data:}
Fuzzy gradual patterns are revisited in \cite{AyouniYLP10} for noisy data where it is often hardly possible to compare attribute values, either because the values are taken from noisy data, or because it is difficult to consider that a small difference between two values is meaningful. An example of a fuzzy gradual pattern could be expressed as ``the closer the age of an employee to 46, the higher his/her income''.


\subsubsection{Extraction gradual itemsets from  temporal data:}

Recently, in \cite{lonlac2018approach,lonlacIDA2020}, the authors proposed an approach to extract gradual patterns in temporal data with an application on paleoecological databases to grasp functional groupings of coevolution of paleoecological indicators that model the evolution of the biodiversity over time. \cite{Ngo2018} introduces a generic method for extracting and analyzing gradual patterns in spatial data at several levels of granularity. The authors apply their method on the Health data to measure  potentially avoidable hospitalization related with both societal and financial issues in public policies.
More recently, in \cite{OwuorLO19} propose fuzzy temporal gradual patterns to integrate the fact that a temporal lag may exist between changes in some attributes and their impact on others. These fuzzy temporal gradual patterns allow to detect the cases of relevant correlations between the attributes of a database whose changes in the value one attribute causes a ripple effect on other attributes with respect to time. \cite{ShahCL19} is interested to extract gradual itemset from property graphs where the attributes of the gradual itemsets are information from the graph, and are retrieved from the graph nodes or relationships.


\section{Mining gradual itemsets with gradual threshold constraints}
\label{sec:approach}

In order to avoid the inconsistent gradual patterns and consider the user defined gradualness, we have modified the gradual patterns mining process by introducing the notion of gradual threshold.

The gradualness or gradual threshold is very related to the domain knowledge. For this observation, we define the gradualness as the Definition \ref{grth}.

\begin{definition}[\textbf{gradual threshold}]Let $\Delta$ be a numerical database and $x$ an attribute of $\Delta$, a gradual threshold is a user defined value $\sigma_x$ such that a variation of $x$ between two tuples $t_1$ and $t_2$ of $\Delta$ considered iff $|t_1.x-t_2.x|\geq \sigma_x$\label{grth}
\end{definition}

Based on the data distribution, we propose to calculate $\sigma_x$ as follows:
\begin{enumerate}
    \item Let $x \in I$ and $sd(x)$ be a standard deviation of the values of the attribute $x$. $\sigma_x$ defined from the distribution of the values of $x$ is called the gradualness threshold of $x$. It  is calculated by the formula \ref{formule} where $k_1$ and $k_2$ are two real numbers .
    \begin{equation}
      \sigma_x = k_1 \times sd(x) + k_2 \label{formule}
    \end{equation}
    \item Let $x \in I$ and $cv(x)$ be the coefficient of variation (relative standard deviation) of the values of the attribute $x$. $\sigma_x$ is calculated from the distribution of the values of $x$ is called the gradualness threshold of $x$ and is defined as follow ($k_1$ and $k_2$ are two real numbers):
    \begin{equation}
      \sigma_x = k_1 \times cv(x) + k_2 \label{formule3}
    \end{equation}
    \item Let $x\in I$ and the component of $x$ sorted in increase order, $\sigma_x$ can also calculated as the standard deviation of the different gaps between two consecutive values of $x$. So $\sigma_x$ is calculated using the equation \ref{formule2}  where $\Delta_{x_i}=t_{i+1}.x-t_{i}.x$ .
    \begin{equation}
        \sigma_x = k_1 \times st(\Delta_{x_i}) + k_2\label{formule2}
    \end{equation}
\end{enumerate}
When $k_1 = k_2 = 0$, gradualness is considered in terms of increasing and decreasing attribute values, which brings back to the case of state of the art approaches \cite{lonlac2018approach, NegrevergneTRM14, Di-JorioLT09}.
\begin{example}
From the Table \ref{exampledataset1}, we can compute the gradual threshold for each numerical attribute. Here, we have consider $k_1=1$ and $k_2=0$.
\begin{table}[!h]
    \begin{center}
    \renewcommand{\arraystretch}{1.5}
       \begin{tabular}{|c|c|c|c|c|}
         \hline
            $x$ & \textbf{$a_1$} & \textbf{$a_2$} & \textbf{$a_3$} & \textbf{$a_4$} \\ 
           
                \hline
                $\sigma_x$ (\ref{formule}) & 0.184 & 0.596 & 0.276 & 0.368 \\
                \hline
                $\sigma_x$ (\ref{formule3}) & 0.141 & 0.189 & 0.103 & 0.157 \\
                \hline
                $\sigma_x$ (\ref{formule2}) & 0.087 & 0.245& 0.113 & 0.189\\
                \hline
            
        \end{tabular}
    \end{center}
    \caption{Example of gradual thresholds from Table \ref{exampledataset1} }\label{tab:2}
 \end{table}
\end{example}

\begin{definition}[Inconsistent gradual pattern]
\label{def:consistence}
Let $G = (i_1^{*_1}, ... , i_k^{*_k})$ be a gradual pattern and $s = \langle t_1, t_2, \ldots, t_n \rangle$ a sequence of objects that satisfy G. $G$ is consistence if and only if $\forall p, 1 \le p \le k$, $\forall j, 1 \le j < n$, $|t_{j+1}.i_p - t_j.i_p| \leq \sigma_p$. Where $\sigma_p$ is the gradual threshold of item $p$.
\end{definition}
As we stated in the introduction section, in certain domain of application such as medicine, this kind of gradual pattern does not provide information to experts, it is considered noisy or incoherent and therefore can negatively influence experts in decision making. It is important to remember that in this area, the requirements for quality of precision are very high because a bad (resp. Good) decision in most cases can kill (resp. Save) a life. In this paper, we are interested in extracting consistent or interesting gradual patterns.

\begin{definition}[Consistent gradual pattern]
Let $G = (i_1^{*_1}, ... , i_k^{*_k})$ be a gradual pattern and $s = \langle t_1, t_2, \ldots, t_n \rangle$ a sequence of objects that satisfy G. $G$ is consistence if and only if $\forall p, 1 \le p \le k$, $\forall j, 1 \le j < n$, $|t_{j+1}.i_p - t_j.i_p| > \sigma_p$. Where $\sigma_p$ is the gradual threshold of item $p$.
\end{definition}

Consider Table \ref{exampledataset1}, and assume that the data respects the time constraint. Using the Table \ref{tab:2}, we can extract the consistent gradual patterns defined in the Table \ref{tab:consistent}
\begin{table}[!h]
    \begin{center}
    \renewcommand{\arraystretch}{1.5}
       \begin{tabular}{|p{3cm}|p{9cm}|}
         \hline
            Using \begin{bf}$\sigma_x$ (\ref{formule}) \end{bf} &
                $(a_{2}^>, a_{3}^>)$;
                $(a_{2}^>, a_{4}^>)$; 
                $(a_{3}^>, a_{4}^>)$;
                $(a_{2}^>, a_{3}^>,a_{4}^>)$\\
            \hline
           Using \begin{bf}$\sigma_x$ (\ref{formule3}) \end{bf} &
                Nothing\\
            \hline
           Using \begin{bf}$\sigma_x$ (\ref{formule2}) \end{bf} &
                $(a_{1}^>, a_{2}^>)$; 
                $(a_{2}^>, a_{3}^>)$\\
            \hline
            
        \end{tabular}
    \end{center}
    \caption{Consistent gradual patterns extracted from Table \ref{exampledataset1} using Table \ref{tab:2} }\label{tab:consistent}
 \end{table}

\subsection{Algorithm}
 The algorithm \ref{alg:grapgt} presents the different steps of the proposed method. This section explains these steps.
     
\begin{algorithm}[H]
    \caption{GRAPGT}
    \label{alg:grapgt}
    \SetAlgoLined
    \KwIn{ \\
        \hspace*{\algorithmicindent}\hspace*{\algorithmicindent}$\Delta$ : a numerical database,\\
        \hspace*{\algorithmicindent}$minSupp$ : a minimum support threshold.}
    \KwOut{\\
        \hspace*{\algorithmicindent}\hspace*{\algorithmicindent}$M$ : frequent gradual patterns.}
        $F \gets SetThreshold(\Delta)$ \;
       
        $\Delta' \gets Num2Cat(\Delta,F)$ \;
        $M \gets MiningAlgo(\Delta', minSupp)$ \; 
        $Return\ M$;
       
\end{algorithm}

The steps of the algorithm \ref{alg:grapgt} are described as following:
\subsubsection{Initialize the gradual threshold : \textbf{\textit{SetThreshold}}}
It corresponds to the step 1 of Algorithm \ref{alg:grapgt}. Depending on the domain knowledge, the expert can set it using his knowledge. In this work, one of the formulas (\ref{formule}), (\ref{formule3}) and (\ref{formule2}) are applied on to the data. For these formulas, we set the values of the two parameters $k_1$ and $ k_2$ equal to 1 and 0 respectively. These values ($k_1=1$ and $k_2=0$) are considered as their default values.


\subsubsection{Transformation the numerical database to categorical database  : \textbf{\textit{Num2Cat}}}
For the algorithms GRITE \cite{Di-JorioLT09} and T-GPatterns presented in \cite{lonlac2018approach}, which first transform the numerical database in categorical one, this threshold is applied during the transformation as following:  
\begin{enumerate}
    \item Case of T-GPatterns approach presented in \cite{lonlac2018approach}

The database $\Delta' = \cal{T'} \times$ $\cal{I'}$ ($|T'|=n-1$ and $|I'|=|I|$) resulting from the application of the $Num2Cat$ function on the numerical database $\Delta = \cal{T} \times$ $\cal{I}$,  is calculated as following : 
\begin{enumerate}
    \item $\forall t_j' \in T', t'_j[i_k]= "+" \iff t_{j+1}[i_k] > t_j[i_k] + \sigma_{i_k}$ 
    \item $\forall t_j' \in T', t'_j[i_k]= "-" \iff t_{j+1}[i_k] < t_j[i_k] - \sigma_{i_k}$
    \item $t'_j[i_k]="o"$ else 
\end{enumerate}
This function allows to generate more symbols "o" than the function \textbf{\textit{Num2Cat}} of the approach \cite{lonlac2018approach}(see Table \ref{tab:3}.a and \ref{tab:3}.b) because graduality is considered if and only if the difference between values of attributes exceed the threshold. Table \ref{tab:3} presents how the numerical database of Table \ref{exampledataset1} is transformed without gradualness threshold application, by application of the formulas (\ref{formule}), formula (\ref{formule2}) and formula (\ref{formule3}) respectively.
\vspace{0.2cm}

\item Case of GRITE algorithm \cite{Di-JorioLT09}. \\
The step "binary matrices generation" of GRITE is modified by the threshold introduction. Let $t_1$ and $t_2$ be two tuples and $x$ an attribute. The binary matrices are calculated as following:
\begin{enumerate}
    \item The matrix $M_1$ of $x^{\geq}$ : $M_{t_1,t_2}=1$ iff $t_2.x-t_1.x\geq \sigma_x$ and 0 otherwise.
     \item The matrix $M_2$ of $x^{\leq}$ : $M_{t_1,t_2}=1$ iff $t_2.x-t_1.x\leq \sigma_x$ else 0.
\end{enumerate}
Like in the first case, the resulting matrices will be less dense than the matrix obtained with the GRITE algorithm. Table \ref{tab:4} shows how the binary matrix is obtained from the numerical database of Table \ref{exampledataset1} respectively without the gradualness threshold on the gradual item $a_1^\geq$, with the gradualness threshold on the gradual item $a_1^\geq$ and the gradual item $a_2^\geq$.
As the computation of the support of a gradual itemset consists in computing the length of the longest path in the graph represented by such binary matrices. We can see from the binary matrices of Table \ref{tab:4} (c) and \ref{tab:4} (d) that introducing gradualness threshold allows to cup the paths in the graph represented by the binary matrices of Table \ref{tab:4} (a) and \ref{tab:4} (b).

It comes out from Table \ref{tab:4} (c) ($M_{a_1^\geq}$ with $\sigma=0.18$) that no path reaches to node $t_2$ and no path begins at node $t_6$. This is not the case of the Table \ref{tab:4} (a) ($M_{a_1^\geq}$ without threshold). Thus, by introducing gradualness threshold, we considerably reduce the search space during the mining process.
The search space is even more reduced between the Table \ref{tab:4} (e) ($M_{(a_1^\geq, a_2^\geq)}$) and Table \ref{tab:4} (f) ($M_{(a_1^\geq, a_2^\geq)}$) with the gradualness threshold $\sigma_{a_1}, \sigma_{a_2}$).
More interesting, the Table \ref{tab:4} (f) can even be reduced by removing object $t_4$ as it is isolated; its line and its column do not have any relation with other objects, it is meaningless in a gradual context. This allows to gain memory, and run-time, as deleted objects are not considered during future joins.
On Table \ref{tab:4} (e), $t_4$ is deleted: all bits from the $t_4$ column and $t_4$ line are set to $0$. Table \ref{tab:4} (g) represents the final matrix.
From Table \ref{tab:4} (g), it is easy to see that the support of the gradual itemset ($a_1^\geq, a_2^\geq$) is equal to $2$ as no path begins at nodes $t_6$, $t_7$, $t_8$ and no path reaches the nodes $t_1$, $t_2$, $t_3$, there is no path containing more than two objects.
From this observation, we can even further reduce Table \ref{tab:4} (g) by removing the lines $t_6$, $t_7$, $t_8$ and the columns $t_1$, $t_2$, $t_3$  to obtain Table \ref{tab:4} (h) with only four lines and three columns.

\end{enumerate}
\begin{table}[!h]
    \begin{minipage}[t]{2cm}
    \begin{center}
       \begin{tabular}{|c|c|c|c|c|}
         \hline
            TID & $a_1$ & $a_2$ & $a_3$ & $a_4$ \\          
                \hline
                $t'_1$  & +& +& +& +\\
                $t'_2$ &-& -& -& -\\
                $t'_3$ &+& +& +& +\\
                $t'_4$ &-& -& -& -\\
                $t'_5$ &+& +& +& +\\
                $t'_6$ &+& +& -& -\\
                $t'_7$ &-& +& +& +\\
                \hline
        \end{tabular}
        \subcaption{}
    \end{center}
    \end{minipage}
    \hspace{1cm}
    \begin{minipage}[t]{2cm}
        \begin{center}
           \begin{tabular}{|c|c|c|c|c|}
             \hline
                TID & $a_1$ & $a_2$ & $a_3$ & $a_4$ \\
                    \hline
                    $t_1'$ &  o &+& o& +\\
                    $t_2'$ &o&o&o& o\\
                    $t_3'$ &o &o& +& o\\
                    $t_4'$ &-& o& -& o\\
                    $t_5'$ &+ &+& + &o\\
                    $t_6'$ &o& o &o &o\\
                    $t_7'$ &-& o &+& +\\
                    \hline
            \end{tabular}
            \subcaption{}
        \end{center}
	\end{minipage}
	\hspace{1cm}
	\begin{minipage}[t]{2cm}
    	\begin{center}
           \begin{tabular}{|c|c|c|c|c|}
             \hline
                TID & $a_1$ & $a_2$ & $a_3$ & $a_4$ \\
                    \hline
                    $t'_1$ & o &+& +& +\\
                    $t'_2$&o& -& -& -\\
                    $t'_3$&o& +& +& o\\
                    $t'_4$&-& -& -& o\\
                    $t'_5$&+& +& +& +\\
                    $t'_6$&+& o& -& o\\
                    $t'_7$&-& +& +& +\\
                    \hline
            \end{tabular}
        \subcaption{}
        \end{center}
	\end{minipage}
	\hspace{1cm}
	\begin{minipage}[t]{2cm}
	  	\begin{center}
           \begin{tabular}{|c|c|c|c|c|}
             \hline
                TID & $a_1$ & $a_2$ & $a_3$ & $a_4$ \\
                    \hline
                    $t'_1$ & + &+& +& +\\
                    $t'_2$&-& -& -& o\\
                    $t'_3$&o& +& +& o\\
                    $t'_4$&-& -& -& o\\
                    $t'_5$&+& +& +& o\\
                    $t'_6$&+& +& -& o\\
                    $t'_7$&-& +& +& +\\
                    \hline
                
            \end{tabular}\\
            \subcaption{}
        \end{center}
    \end{minipage}\\
    \caption{Transformed database obtained from Table \ref{exampledataset1} : (a) without threshold, (b) with threshold (\ref{formule}), (c) with threshold (\ref{formule3}) and (d) with threshold (\ref{formule2}) respectively}\label{tab:3}
\end{table}

\begin{table}[!h]
    \begin{minipage}[t]{3cm}
    \begin{center}
       \begin{tabular}{|c|c|c|c|c|c|c|c|c|}
         \hline
           $\Rsh$  & $t_1$ & $t_2$ & $t_3$ & $t_4$ & $t_5$ & $t_6$ & $t_7$ & $t_8$ \\
                \hline
                $t_1$ & 0 & \textbf{1} & 1 & 1 & 0 & 1 & 1 & 1\\
                $t_2$ & 0 & \textbf{0} & 0 & 0 & 0 & 1 & 1 & 0\\
                $t_3$ & 0 & \textbf{1} & 0 & 1 & 0 & 1 & 1 & 1\\
                $t_4$ & 0 & \textbf{1} & 0 & 0 & 0 & 1 & 1 & 0\\
                $t_5$ & 1 & \textbf{1} & 1 & 1 & 0 & 1 & 1 & 1\\
                $t_6$ & \textbf{0} & \textbf{0} & \textbf{0} & \textbf{0} & \textbf{0} & \textbf{0} & \textbf{1} & \textbf{0} \\
                $t_7$ & 0 & \textbf{0} & 0 & 0 & 0 & 0 & 0 & 0\\
                $t_8$ & 0 & \textbf{1} & 0 & 1 & 0 & 1 & 1 & 0\\
                \hline
        \end{tabular}
        \subcaption{$M_{a_1^\geq}$}
    \end{center}
    \end{minipage}
    \hspace{1cm}
        \begin{minipage}[t]{3cm}
    \begin{center}
       \begin{tabular}{|c|c|c|c|c|c|c|c|c|}
         \hline
            $\Rsh$  & $t_1$ & $t_2$ & $t_3$ & $t_4$ & $t_5$ & $t_6$ & $t_7$ & $t_8$ \\
                \hline
                $t_1$ & 0 & 1 & 1 & 1 & 1 & 1 & 1 & 1\\
                $t_2$ & 0 & 0 & 0 & 1 & 0 & 1 & 1 & 1\\
                $t_3$ & 0 & 1 & 0 & 1 & 1 & 1 & 1 & 1\\
                $t_4$ & 0 & 0 & 0 & 0 & 0 & 1 & 1 & 1\\
                $t_5$ & 0 & 1 & 0 & 1 & 0 & 1 & 1 & 1\\
                $t_6$ & 0 & 0 & 0 & 0 & 0 & 0 & 1 & 1\\
                $t_7$ & 0 & 0 & 0 & 0 & 0 & 0 & 0 & 1\\
                $t_8$ & 0 & 0 & 0 & 0 & 0 & 0 & 0 & 0\\
                \hline
        \end{tabular}
        \subcaption{$M_{a_2^\geq}$}
    \end{center}
    \end{minipage}
    \hspace{1cm}
 \begin{minipage}[t]{3cm}
    \begin{center}
       \begin{tabular}{|c|c|c|c|c|c|c|c|c|}
         \hline
            $\Rsh$  & $t_1$ & $t_2$ & $t_3$ & $t_4$ & $t_5$ & $t_6$ & $t_7$ & $t_8$ \\
                \hline
                $t_1$ & 0 & 1 & 1 & 1 & 0 & 1 & 1 & 1\\
                $t_2$ & 0 & \textbf{0} & 0 & 0 & 0 & 1 & 1 & 0\\
                $t_3$ & 0 & 1 & 0 & 1 & 0 & 1 & 1 & 1\\
                $t_4$ & 0 & 0 & 0 & 0 & 0 & 1 & 1 & 0\\
                $t_5$ & 0 & 1 & 0 & 1 & 0 & 1 & 1 & 1\\
                $t_6$ & 0 & 0 & 0 & 0 & 0 & 0 & 1 & 0 \\
                $t_7$ & \textbf{0} & \textbf{0} & \textbf{0} & \textbf{0} & \textbf{0} & \textbf{0} & \textbf{0} & \textbf{0} \\
                $t_8$ & \textbf{0} & \textbf{0} & \textbf{0} & \textbf{0} & \textbf{0} & \textbf{0} & \textbf{0} & \textbf{0} \\
                \hline
        \end{tabular}
        \subcaption{$M_{(a_1^\geq, a_2^\geq)}$}
    \end{center}
    \end{minipage}
    \hspace{1cm}
    \begin{minipage}[t]{3cm}
        \begin{center}
           \begin{tabular}{|c|c|c|c|c|c|c|c|c|}
             \hline
                $\Rsh$  & $t_1$ & \textbf{$t_2$} & $t_3$ & $t_4$ & $t_5$ & $t_6$ & $t_7$ & $t_8$ \\
                \hline
                $t_1$ & 0 & \textbf{0} & 0 & 0 & 0 & 1 & 1 & 0\\
                $t_2$ & 0 & \textbf{0} & 0 & 0 & 0 & 0 & 1 & 0\\
                $t_3$ & 0 & \textbf{0} & 0 & 0 & 0 & 1 & 1 & 0\\
                $t_4$ & 0 & \textbf{0} & 0 & 0 & 0 & 1 & 1 & 0\\
                $t_5$ & 0 & \textbf{0} & 1 & 1 & 0 & 1 & 1 & 1\\
                $t_6$ & \textbf{0} & \textbf{0} & \textbf{0} & \textbf{0} & \textbf{0} & \textbf{0} & \textbf{0} & \textbf{0} \\
                $t_7$ & 0 & \textbf{0} & 0 & 0 & 0 & 0 & 0 & 0\\
                $t_8$ & 0 & \textbf{0} & 0 & 0 & 0 & 1 & 1 & 0\\
                    \hline
            \end{tabular}
            \subcaption{$M_{a_1^\geq}$ with $\sigma=0.18$}
        \end{center}
	\end{minipage}
	\vspace{1cm}
	\begin{minipage}[t]{3cm}
    	\begin{center}
         \begin{tabular}{|c|c|c|c|c|c|c|c|c|}
             \hline
                $\Rsh$  & $t_1$ & $t_2$ & $t_3$ & $t_4$ & $t_5$ & $t_6$ & $t_7$ & $t_8$ \\
                \hline
                $t_1$ & 0 & 1 & 0 & 1 & 0 & 1 & 1 & 1\\
                $t_2$ & 0 & 0 & 0 & 0 & 0 & 0 & 1 & 1\\
                $t_3$ & 0 & 0 & 0 & 1 & 0 & 1 & 1 & 1\\
                $t_4$ & 0 & 0 & 0 & 0 & 0 & 0 & 0 & 1\\
                $t_5$ & 0 & 0 & 0 & 0 & 0 & 1 & 1 & 1\\
                $t_6$ & \textbf{0} & \textbf{0} & \textbf{0} & \textbf{0} & \textbf{0} & \textbf{0} & \textbf{0} & \textbf{0} \\
                $t_7$ & \textbf{0} & \textbf{0} & \textbf{0} & \textbf{0} & \textbf{0} & \textbf{0} & \textbf{0} & \textbf{0} \\
                $t_8$ & \textbf{0} & \textbf{0} & \textbf{0} & \textbf{0} & \textbf{0} & \textbf{0} & \textbf{0} & \textbf{0} \\
                \hline
            \end{tabular}
         \subcaption{$M_{a_2^\geq}$ with $\sigma=0.59$}
        \end{center}
	\end{minipage}
        \begin{minipage}[t]{3cm}
        \begin{center}
           \begin{tabular}{|c|c|c|c|c|c|c|c|c|}
             \hline
                $\Rsh$  & $t_1$ & \textbf{$t_2$} & $t_3$ & $t_4$ & $t_5$ & $t_6$ & $t_7$ & $t_8$ \\
                \hline
                $t_1$ & 0 & 0 & 0 & \textbf{0} & 0 & 1 & 1 & 0\\
                $t_2$ & 0 & 0 & 0 & \textbf{0} & 0 & 0 & 1 & 0\\
                $t_3$ & 0 & 0 & 0 & \textbf{0} & 0 & 1 & 1 & 0\\
                $t_4$ & \textbf{0} & \textbf{0} & \textbf{0} & \textbf{0} & \textbf{0} & \textbf{0} & \textbf{0} & \textbf{0} \\
                $t_5$ & 0 & 0 & 0 & \textbf{0} & 0 & 1 & 1 & 1\\
                $t_6$ & \textbf{0} & \textbf{0} & \textbf{0} & \textbf{0} & \textbf{0} & \textbf{0} & \textbf{0} & \textbf{0} \\
                $t_7$ & \textbf{0} & \textbf{0} & \textbf{0} & \textbf{0} & \textbf{0} & \textbf{0} & \textbf{0} & \textbf{0} \\
                $t_8$ & \textbf{0} & \textbf{0} & \textbf{0} & \textbf{0} & \textbf{0} & \textbf{0} & \textbf{0} & \textbf{0} \\
                    \hline
            \end{tabular}
            \subcaption{$M_{(a_1^\geq, a_2^\geq)}$ with $\sigma_{a_1}, \sigma_{a_2} $}
        \end{center}
	\end{minipage}
        \begin{minipage}[t]{3cm}
        \begin{center}
           \begin{tabular}{|c|c|c|c|c|c|c|c|}
             \hline
                $\Rsh$  & $t_1$ & \textbf{$t_2$} & $t_3$ & $t_5$ & $t_6$ & $t_7$ & $t_8$ \\
                \hline
                $t_1$ & 0 & 0 & 0 &  0 & 1 & 1 & 0\\
                $t_2$ & 0 & 0 & 0 &  0 & 0 & 1 & 0\\
                $t_3$ & 0 & 0 & 0 & 0 & 1 & 1 & 0\\
                $t_5$ & 0 & 0 & 0 & 0 & 1 & 1 & 1\\
                $t_6$ & \textbf{0} & \textbf{0} & \textbf{0} & \textbf{0} & \textbf{0} & \textbf{0} & \textbf{0} \\
                $t_7$ & \textbf{0} & \textbf{0} & \textbf{0} &  \textbf{0} & \textbf{0} & \textbf{0} & \textbf{0} \\
                $t_8$ & \textbf{0} & \textbf{0} & \textbf{0} &  \textbf{0} & \textbf{0} & \textbf{0} & \textbf{0} \\
                \hline
            \end{tabular}
             \subcaption{reduced $M_{(a_1^\geq, a_2^\geq)}$ with $\sigma_{a_1}, \sigma_{a_2} $}
        \end{center}
	\end{minipage}
	\hspace{1cm}
	  \begin{minipage}[t]{3cm}
        \begin{center}
           \begin{tabular}{|c|c|c|c|}
             \hline
                $\Rsh$  & $t_6$ & $t_7$ & $t_8$ \\
                \hline
                $t_1$ & 1 & 1 & 0\\
                $t_2$ &  0 & 1 & 0\\
                $t_3$ & 1 & 1 & 0\\
                $t_5$ & 1 & 1 & 1\\
                \hline
            \end{tabular}
            \subcaption{reduced $M_{(a_1^\geq, a_2^\geq)}$ with $\sigma_{a_1}, \sigma_{a_2} $}
        \end{center}
	\end{minipage}
    \caption{Binary matrix of $a_1^\geq$ and $a_2^\geq$ (respectively ($a_1^\geq, a_2^\geq$))  without threshold, and those of $a_1^\geq$ and $a_2^\geq$ (respectively ($a_1^\geq, a_2^\geq$)) with the threshold defined by equation (\ref{formule}) from Table \ref{exampledataset1}.}\label{tab:4}
\end{table}

\subsubsection{Gradual itemsets mining : \textbf{\textit{MiningAlgo}}}
For T-GPatterns \cite{lonlac2018approach}, the remainded step is the procedure $searchCoevolution(Apriori(\Delta', minSupp))$.  In our proposal, this step is done exactly as proposed by the authors.
Introducing the threshold constraints in GRITE or Paraminer algorithm consists of changing the processing step where the binary matrix associated to each attribute is calculated. All other steps (\textit{initialize, AND operator and DeleteAloneTuples}) are not changed.
\subsection{GRAPGT algorithm Proprieties}

\textbf{Correctness}.
As the GRAPGT algorithm is based on the existing algorithms which is prove to be correct, the threshold introduction step added permits to make more sparse the the transformed matrix, input of the frequent itemsets mining algorithm (Apriori-like). As these algorithms are prove to be correct, GRAPGT is correct too.

\textbf{Completeness}.
As the GRAPGT algorithm is based on the prove complet algorithm, the introduction threshold step does not change the this property for the obtained algorithm. The main consequence is the reduction of the gradual frequent itemsets due to the threshold. 

\textbf{Complexity}. If the threshold of the different attributes are set by the user, the theoretical complexity remains that of the chose algorithm (GRITE, T-GPatterns,...). But if these thresholds are calculated by the formulas (\ref{formule}), (\ref{formule3}), or (\ref{formule2}) , the time complexity is impacted by $n\times m$ where $n$ is the row number and $m$ the columns number. But as these algorithms are Apriori-based, this theoretical complexity remains exponential ($2^m$), the same obtained when $k_1=k_2=0$.

\subsection{Impact of the threshold}
After application of the gradualness threshold on the data, some aspects of the expected results are impacted.
\subsubsection{On the resulted gradual itemsets}
\begin{proposition}
Let $\Delta$ be a numerical database, $G_1$ a set of frequent gradual itemsets obtained using the algorithm T-GPatterns or GRITE, and $G_2$ a set of gradual itemsets obtained by our approach it is true that $G_2\subseteq G_1$
\end{proposition}

\begin{proof}
The main effect of the gradualness threshold to introduce more "o" character in the intermediate transformed matrix. Let $M$ (resp. $M'$) be a intermediate binary obtained by T-GPatterns without gradualness threshold (resp. with gradualness threshold), $\forall_{ij}$ if $M'_{i,j}=1$ then $M_{ij}=1$ because $M$ is calculated with $\sigma=0$. With this observation and due to the fact that $M$ is more dense than $M'$, all frequent extracted from $M'$ is also frequent in $M$. This implies that $G_2\subseteq G_1$.
\end{proof}

\subsubsection{Relationship between closed frequent gradual itemsets}
\begin{proposition}
Let $\Delta$ be a numerical database and $x$ an attribute of $\Delta$ such that there exists a row $i$ and $M_{i,x}=1$ and $M'_{i,x}=0$; $M$ (resp. $M'$) is the intermediate binary matrix by the transformation of $\Delta$ without (resp. with) gradualness threshold application.

If $X$is a closed gradual itemset in $M$ such that $x\in X$ then $X-\{x\}$ is closed gradual itemset in $M'$
\end{proposition}

\begin{proof}
From the operation of the closure defined in \cite{AyouniLYP10}, gradual pattern $X$ and its extension forms a maximal rectangle in $M$ (rectangle of 1 value). So the 'o' character put at the column $x$ in the matrix $M'$ (due to the threshold application) consist in deleting a column $x$ in the maximal rectangle $(g(X),X)$; the remained rectangle $(g(X-\{x\}),X-\{x\})$ is maximal, then $X-\{x\}$ is closed.
\end{proof}
This proposition is also true if $x\in X$ is replaced by $X'\subset X$.
\subsubsection{On the support of extracted gradual itemsets}

It is clear that after applying the graduality threshold, the support of some gradual itemsets decreases. So during exploration, this fact should be taken into account.

\begin{property}
\label{pro:relationshipBetweenSupportThreshold}
Let $\Delta$ be a numerical database and $x$ an attribute of $\Delta$. Let $minSupp$ a minimum support threshold and $\sigma$ a gradualness threshold set by a user. For the algorithm T-GPatterns, the following holds: \\ if $\sum_{i=0}^n |t_{i+1}.x-t_{i}.x| < \sigma \times minSupp$ then $x^* ~(* \in \{\leq, \geq\})$ as well as all of its supersets are not frequent in $\Delta$.
\end{property}

\begin{proof}
Assume $\sum_{i=0}^n |t_{i+1}.x-t_{i}.x| < \sigma \times minSupp$ , $x^* ~(* \in \{\leq, \geq\})$. Suppose that $x^*$ is a frequent gradual item, according to the gradual itemset definition proposed in \cite{lonlac2018approach}, this means that there is a list of sequence of consecutive objects $s = \langle s_1, \ldots, s_k \rangle $ such that $|t_{i+1}.x - t_i.x| \geq \sigma$, for $t_i, t_{i+1} \in s_j$ $(1 \leq j \leq k)$ and $\sum_{j=1}^k |s_j| \geq minSupp$.
So $\sum_{j=1}^k \sum_{i=1}^{|s_j|-1} |t_{i+1}.x-t_{i}.x| \geq \sigma \times minSupp$.
So, $\sum_{i=0}^n |t_{i+1}.x-t_{i}.x| \geq \sigma \times minSupp$. which contradicts the initial hypothesis.
\end{proof}

Property \ref{pro:relationshipBetweenSupportThreshold} allows to perform dataset reduction before  mining process. Hence the whole dataset is not required to compute candidate patterns for a given minimum support and gradualness threshold. The search space containing $x$ and all its supersets can be discarded.

\section{Experiments}
\label{sec:expe}
We present in this section an experimental study of the execution time and number of gradual patterns extracted using \textit{GRAPGT}.
We also evaluate the performance in terms of memory usage.
It should be recalled here that the issue of the management of the quantity of mined patterns is a great challenge as in many practical applications, the number of patterns mined can prove to be intractable for user-defined frequency threshold. All the experiments are conduced on a computer with 8 GB of RAM, of processor Intel(R) Core(TM) i5-8250U. We compare first of all our R implementation of \textit{GRAPGT}  with the original R implementation of T-Gpatterns \cite{lonlac2018approach}, and secondly the C++ implementation of \textit{GRAPGT} with the original c++ implementation of ParaMiner \cite{NegrevergneTRM14}.
\subsection{Source code}
The source code of our proposed algorithm \textit{GRAPGT} (respectively \textit{GRAPGT-ParaMiner}) can be obtained from \url{https://github.com/chirmike/GRAPGT} (respectively \url{https://github.com/Chirmeni/GRAPGT-Paraminer}).

\subsection{Data Description}
Table \ref{dataset} presents the characteristics of the datasets used in the experiments for evaluating the performance of our proposed algorithm

 \begin{table}[!h]
    \begin{center}
       \begin{tabular}{|c|c|c|c|c|}
         \hline
             Dataset & \#objects &\#attributes & Domain &Origin \\
           
                \hline
               Paleo & 111 &87& Paleo-ecology& \cite{lonlac2018approach}\\
               Cancer& 410& 30& Medical&\cite{UCI}\\
               Air quality&9358& 13& ecology& \cite{UCI}\\
               paraMiner-data &109& 4413& Synthetic&\cite{NegrevergneTRM14} \\
                \hline
        \end{tabular}
    \end{center}
    \caption{Experimentation datasets }\label{dataset}
 \end{table}
The datasets described in Table \ref{dataset} are numerical databases: the first numerical database used is a temporal database (database with a temporal order among objects (or rows)) of paleoecological indicators \cite{lonlac2018approach} from Lake Aydat located in the southern part of the Chane des Puys volcanic chain (Massif Central in France); this database contains $111$ objects corresponding to different dates identified on the lake record considered, and $87$ attributes corresponding to different indicators of paleoecological anthropization (pollen grains) (cf. \cite{lonlac2018approach} for more details). 

The second numerical database is a temporal database obtained from \texttt{UCI Machine Learning repository} \cite{UCI} describing the hourly averaged responses from an array of $5$ metal oxide chemical sensors embedded in an Air Quality Chemical Multisensor Device; this database contain $9358$ instances corresponding to the hourly averaged responses and $13$ attributes corresponding to the ground truth hourly averaged concentrations for CO, Non Metanic Hydrocarbons, Benzene, Total Nitrogen Oxides (NOx) and Nitrogen Dioxide (NO2) and any more. 

The third database is a non temporal Cancer database also taken from \texttt{UCI Machine Learning repository} describing characteristics of the cell nuclei computed from a digitized image of a fine needle aspirate (FNA) of a breast mass; this database used for experiments in Paraminer \cite{NegrevergneTRM14}  containing $410$ objects described by $30$ attributes. The fourth numerical database is a synthetic non-temporal database from Paraminer's database containing $109$ objects and $4413$ attributes. 

\subsection{Results}
\label{subsec:res}
The results of the experiment study on the four databases demonstrate the importance of incorporating the gradual threshold into the data mining process and the significant benefits it provides. This taking into account of gradual threshold in the mining process not only allows users to provide consistent gradual patterns (see Definition \ref{def:consistence}) extracts imprecise numerical databases, but also to reduce on average by $86.06\%$ the quantity of frequent gradual patterns to be analyzed, by $84.21\%$ the extraction time of these frequent gradual pattern and finally, by $50.01\%$ the size of the memory consumed. Furthermore, the gradual threshold removes noisy / inconsistent gradual patterns.

We present these results in two steps : first of all, we present results obtained on the non-temporal databases (Fig. \ref{fig:cancer} and \ref{fig:Para}), and secondly, we present results obtained on temporal databases (Fig. \ref{fig:paleo} and \ref{fig:air}). We varied the support threshold in the interval $[0.1, 0.5]$ with a step of $0.1$. Throughout our experiments, we have set $k_1$ to $1$ and $k_2$ to $0$.\\

\subsubsection{Comparative experiments : non-temporal databases.}
The first experiment compares the execution time, the memory usage and the number of frequent gradual patterns for GRAPGT and ParaMiner on non temporal databases Cancer and ParaMiner-Data (cf Table \ref{dataset}).  
\begin{figure}[!h]
    \begin{minipage}[t]{3cm}
		\centering
		\includegraphics[scale=0.18]{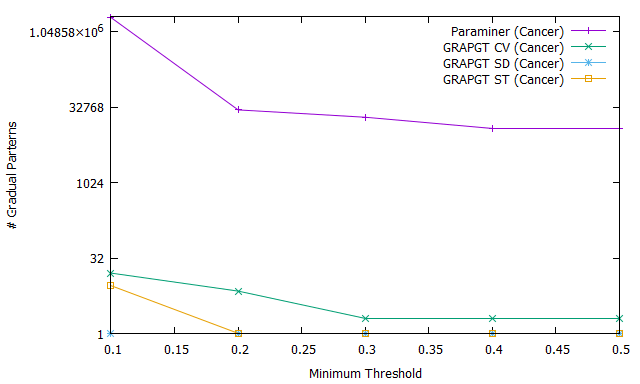}
		\subcaption{}
	\end{minipage}
	\hspace{1cm}
	\begin{minipage}[t]{3cm}
		\centering
		\includegraphics[scale=0.18]{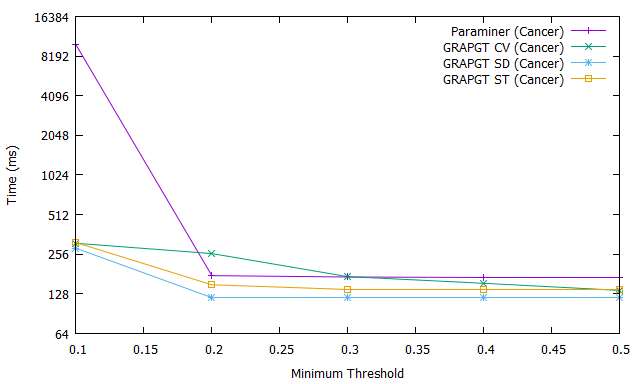}
		\subcaption{} 
  \end{minipage}
  \hspace{1cm}
      \begin{minipage}[t]{3cm}
		\centering
		\includegraphics[scale=0.18]{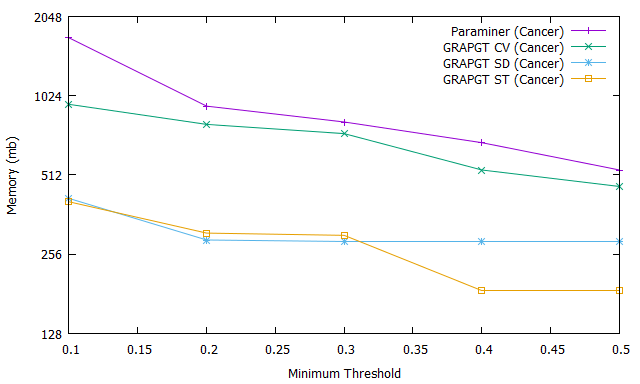}
		\subcaption{}
	\end{minipage}
	
  \caption{Comparative study in number of frequent gradual patterns (a), in time (b) and in memory usage (c) of four algorithms on the cancer database.}\label{fig:cancer}
\end{figure}
Fig. \ref{fig:cancer}.a shows the number of frequent gradual patterns for ParaMiner algorithm (purple curve), \textit{GRAPGT} CV algorithm (green curve), \textit{GRAPGT} SD algorithm (blue curve), and \textit{GRAPGT} ST algorithm (orange curve) on the cancer database when varying the support, with a reduced number of attributes. Figure \ref{fig:cancer}.b shows the runtime of all these algorithms on the same cancer database when varying support, while Figure \ref{fig:cancer}.c, shows the memory usage. 

It should be like to remind you that \textit{GRAPGT} CV, SD and ST are modified versions of the ParaMiner algorithm to which we have integrated different gradual thresholds (coefficient of variation, standard deviation and standard deviation of the deviations).   
The number of frequents gradual patterns results show that \textit{GRAPGT} considerably reduces this number than ParaMiner when database have a small attributes : for a value of support threshold equal to $0.2$ it extracts $7$ frequent gradual patterns when the gradual threshold is defined by equation or formule (\ref{formule3}), while ParaMiner extracts $28761$.  This is a considerable advantage that should not be overlooked for the end user (the expert), because it is easier for an expert to analyze $7$ frequent  gradual patterns at  the expense of $28761$; in addition, frequent gradual patterns extracted with our approach have an additional information : they are gradual patterns with a strong graduality power. The execution time results show that \textit{GRAPGT} is faster that Paraminer for handling this database with small attributes when the support threshold increases : for a value of support threshold equal to $0.1$ it answers in $330$ ms when coefficient of variation considered like gradual threshold, while Paraminer needs $3$ hours and $16$ minutes.
\textit{GRAPGT} remains better in memory used compared to ParaMiner. Most the support is great, less execution time, number and memory are great.

\begin{figure}[!h]
    \begin{minipage}[t]{3cm}
		\centering
		\includegraphics[scale=0.18]{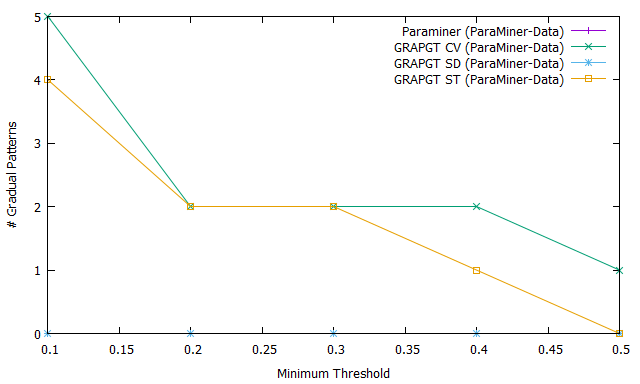}
		\subcaption{}
	\end{minipage}
	\hspace{1cm}
	\begin{minipage}[t]{3cm}
		\centering
		\includegraphics[scale=0.18]{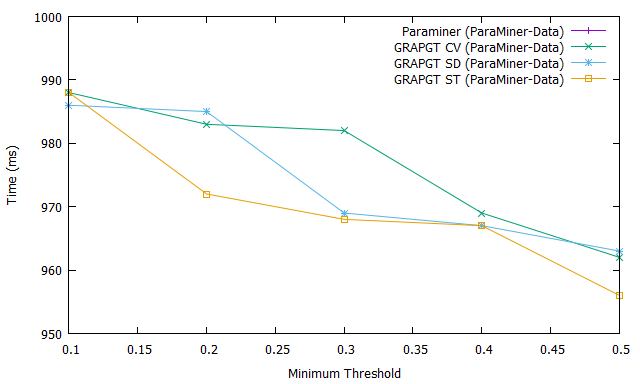}
		\subcaption{} 
  \end{minipage}
  \hspace{1cm}
	\begin{minipage}[t]{3cm}
		\centering
		\includegraphics[scale=0.18]{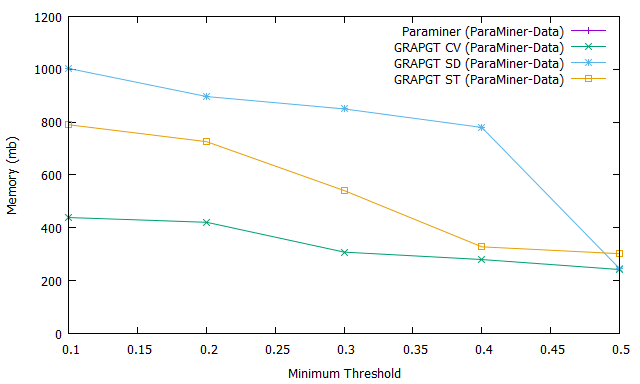}
		\subcaption{} 
  \end{minipage}
	\caption{Comparative study in number of frequent gradual patterns (a), in time (b) and in memory usage (c) of four algorithms on the synthetic database}\label{fig:Para}
\end{figure}

Fig. \ref{fig:Para} shows the number of frequent gradual patterns extracted (Fig. \ref{fig:Para}.a), the execution time (Fig. \ref{fig:Para}.b) and the memory usage (Fig.  \ref{fig:Para}.c) for previous four algorithms when varying the support on numerical database ParaMiner-Data. We can see that figures Fig. \ref{fig:Para}.a, \ref{fig:Para}.b and \ref{fig:Para}.c each have $3$ curves instead of $4$. This observation is due to the fact that the ParaMiner algorithm without gradual threshold does not run until the end (it crashes memory) on the ParaMiner-Data database containing $109$ transactions and $4413$ attributes. This database is more dense than the previous one, as the complexity lies in the number of attributes which determines the number of frequent gradual patterns (problem of combinatorial explosion).

However, when we introduce the graduality threshold, we manage to execute and extract a reduced number of frequent gradual patterns (e.g. $5$ when the coefficient of variation is considered as gradual threshold for a support equal to $0.1$). This result shows that our \textit{GRAPGT} approach is scalable compared to ParaMiner because it gives a result where Paraminer fails. It goes without saying that the introduction of the graduality threshold considerably reduces the execution time, the memory usage and the number of frequent gradual patterns on databases with a small or large number of attributes ; in addition, these results show that it is easier for the end user or the expert to analyze the frequent gradual patterns extracted from approaches taking the gradual threshold into account than the approaches not taking this threshold into account.

\subsubsection{Comparative experiments : temporal databases.}
The next experiment compares the execution time, the memory usage and the number of frequent gradual patterns for \textit{GRAPGT} and \textit{T-GPatterns} on temporal databases Paleo and Air quality (cf Table \ref{dataset}). It is good to remember that the choice to compare our approach with the approaches of extraction of frequent gradual patterns on numerical temporal databases is due to the fact that the latter aims at the same objective as us : extract useful and relevant informations from a numerical database that respect certain constraints of the application domain. Taking these constraints into account has a positive impact for users. It makes it possible to reduce the number of frequent gradual patterns facilitating their analysis by the users.
\begin{figure}[!h]
    \begin{minipage}[t]{3cm}
		\centering
		\includegraphics[scale=0.18]{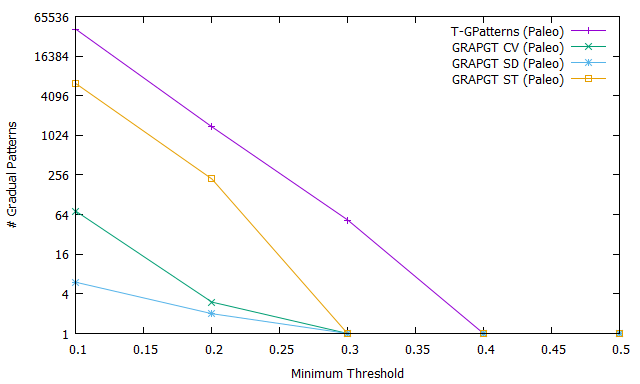}
		\subcaption{}
	\end{minipage}
	\hspace{1cm}
	\begin{minipage}[t]{3cm}
		\centering
		\includegraphics[scale=0.18]{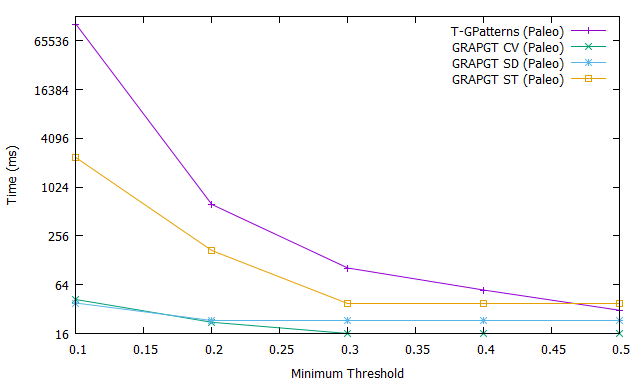}
		\subcaption{} 
  \end{minipage}
  \hspace{1cm}
	\begin{minipage}[t]{3cm}
		\centering
		\includegraphics[scale=0.18]{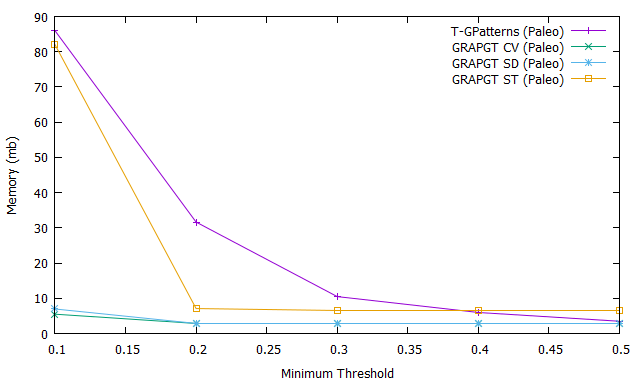}
		\subcaption{} 
  \end{minipage}
  
	\caption{Comparative study in number of frequent gradual patterns (a), in time (b) and in memory usage (c) of four algorithms on the Paleoecologique database.}\label{fig:paleo}
\end{figure}

Thus, Fig \ref{fig:paleo}.a shows the number of frequent gradual patterns for \textit{T-GPatterns} algorithm (purple curve), \textit{GRAPGT} CV algorithm (green curve), \textit{GRAPGT} SD algorithm (blue curve), and \textit{GRAPGT} ST algorithm (orange curve) on the paleoecology database when varying the support. In this Figure, when minimum support  is equal to $0.1$, the number of frequent gradual patterns extract by T-GPatterns without take into account  gradual threshold is equal to $41867$; this number is  considerably reduce when the gradual threshold is taken into account : it is equal to $72$ when the gradual threshold is define by equation \ref{formule} (\textit{GRAPGT} CV), $6$ when the gradual is defined by equation \ref{formule3} (\textit{GRAPGT} SD) and $6255$ when the gradual threshold is defined by equation \ref{formule2} (\textit{GRAPGT} ST). Similarly, when the support increases, the number of frequent gradual patterns extracted from T-GPatterns remains greater than the number extracted from \textit{GRAPGT} CV, \textit{GRAPGT} SD and \textit{GRAPGT} ST. In this context, these results show that it is easier for an user or an expert to analyze the frequent gradual patterns extracted from approaches taking the graduality threshold into account than the approaches not taking this threshold into account.

Fig. \ref{fig:paleo}.b shows the evolution of the extraction time of the frequent gradual patterns of each algorithm as a function of the variation of the support. Here the end users will also save time if they integrate the gradual threshold into the mining process; this is easily seen in this Fig. \ref{fig:paleo}.b. Fig. \ref{fig:paleo}.c shows the memory usage for these four algorithms when varying support. It is easy to see that our \textit{GRAPGT} algorithm consumes less memory than the T-GPatterns algorithm.
\begin{figure}[!h]
    \begin{minipage}[t]{3cm}
		\centering
		\includegraphics[scale=0.18]{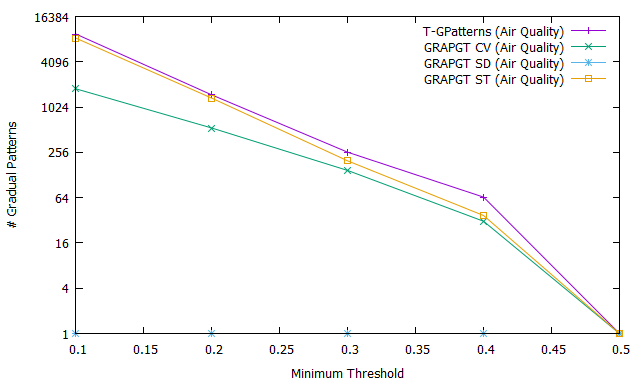}
		\subcaption{}
	\end{minipage}
	\hspace{1cm}
	\begin{minipage}[t]{3cm}
		\centering
		\includegraphics[scale=0.18]{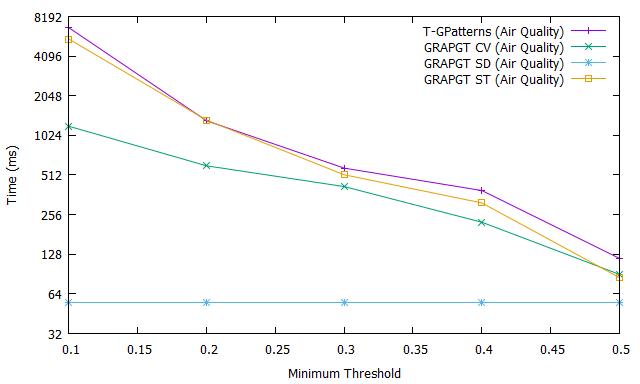}
		\subcaption{}
  \end{minipage}
  \hspace{1cm}
  \begin{minipage}[t]{3cm}
		\centering
		\includegraphics[scale=0.18]{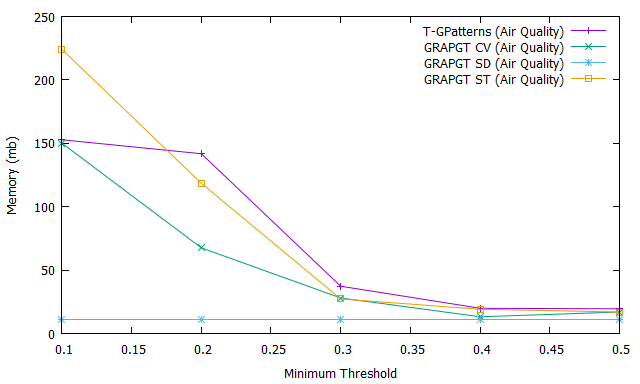}
		\subcaption{} 
  \end{minipage}
	\caption{Comparative study in number of frequent gradual patterns (a), in time (b) and in memory usage (c) of four algorithms on the Air quality database.}\label{fig:air}
\end{figure}

Fig \ref{fig:air} shows the evolution of the number (Fig. \ref{fig:air}.a), the execution time (Fig. \ref{fig:air}.b) and the memory usage (Fig. \ref{fig:air}.c) of extraction frequent gradual patterns as a function of the variation of the support threshold after application of the T-GPatterns \cite{lonlac2018approach} algorithm (purple curve), \textit{GRAPGT} CV algorithm (green curve), \textit{GRAPGT} SD algorithm (blue curve), and GRAPGT ST algorithm (orange curve) on the air quality database. Unlike the paleoecological database, this database has a large number of attributes ($9358$).

The objective here is to show the impact of the gradual threshold on databases with a large number of attributes. We can easily see that the curves of the graph are monotonically decreasing which means that the number of frequent gradual patterns decreases when the support increases. In addition, the introduction of the gradual threshold considerably reduces the extraction time and number of frequent gradual patterns.

\subsubsection{Consistent gradual patterns : }
In this section, we present a list of the consistent (cf. Definition \ref{def:consistence}) gradual patterns that we were able to extract in the different databases that we presented above. We recall that a gradual pattern is said to be consistent if it has a great power of variation (increase or decrease).
\begin{table}[!h]
    \begin{center}
    \renewcommand{\arraystretch}{1.5}
       \begin{tabular}{|p{3cm}|p{9cm}|}
         \hline
             \textbf{Datasets} & \textbf{Consistent gradual patterns} \\
         \hline
           Paleo &
                $Corylus^< Quercus^>$;
                $Quercus^< Fagus^>$; 
                $Quercus^< Poaceae^>$
                \\
            \hline
           Cancer &
                $perimeter_mean^> area_mean^>$; 
                $area_mean^> texture_worst^>$;
                $radius_worst^> texture_worst^>$
                \\
            \hline
           Air quality &
                $PT08.S1^> C_6H_6^>$; 
                $PT08.S1^> C_6H_6^> PT08.S2^>$;
                $PT08.S1^> C_6H_6^> PT08.S2^> NO2^>$
                \\
            \hline
          ParaMiner-data &
                $4411^> 4412^>$; 
                $4411^> 4412^<$
                \\
            \hline
        \end{tabular}
    \end{center}
    \caption{Consistent gradual patterns extracted from Table \ref{exampledataset1} using Table \ref{tab:2} }\label{tab:consistent}
 \end{table}
 

\section{Discussion}
\label{sec:discussion}
\subsection{Memory space}
The memory space used to keep the threshold vectors seems to augment the memory needs of the algorithm. This does not change the space complexity of the algorithm. But in fact during the experiment, we observe in some cases existing algorithms (ParaMiner) does not execute due to the memory space insufficiency while with the threshold introduction, our method runs and out some itemsets.
\subsection{Execution time and complexity}
The imprecision processing by the threshold introduction augments one step in the algorithm. Theoretically, the overall algorithm complexity does not change. The imprecision processing could also augments the execution time; but it is not the case. The introduced threshold permits to reduce the density of the intermediate (binary) matrix, which permit to reduce the execution time of the remained steps. This explains why for all datasets used in the execution time of the GRAPT algorithm is always lower than other algorithms.
\subsection{Extracted gradual itemsets}
The GRAPGT proposed due to the threshold introduction provides less gradual itemsets than the classical method. It then eliminates some gradual itemsets whose interpretation could also be useful for an expert. But the fact that the number of resulted gradual itemsets is lower, facilitates the expert interpretation work. So instead of providing a million of itemsets, a few are presented to the expert and he can easily manage it.
\section{Conclusion and Perspectives}
\label{sec:conclu}
This paper proposes an approach to automatically extract gradual patterns by taking into account the user preferences about each attribute of the database during the mining process.
An algorithm named  \textit{GRAPGT (GRAdual Patterns with Gradualness Threshold)} based on the integration of the constraints of variation threshold from which to consider a gradualness (increase/decrease) into traditional gradual patterns mining algorithms  was proposed for extracting these patterns to avoid drawbacks of traditional gradual patterns mining algorithms on some data like noisy data; the graduality between two objects is no longer considered simply in terms of increase or decrease, but it is considered if and only if, the difference in attribute value between two objects is greater than a certain quantity (gradual threshold).
 Experimental results obtained on several real world databases have shown that the introduction of the gradual threshold in the gradual pattern mining process not only significantly reduces the amount of frequent gradual patterns to be analyzed, but also saves considerable time and memory.
 The proposed algorithm can returns a small set of gradual patterns to the user while filtering many  patterns that are not meeting specific gradualness requirements. Moreover, it also allows the generation of gradual patterns on certain large databases where some algorithms in the literature fail for the reason of the search space very huge as show the experimental results (\textit{ParaMiner} run on the ParaMiner-Data database fail). 
 However, it would be interesting to study the impact of the gradualness threshold on the quality of frequent gradual patterns, and also on the choice of support threshold. This last point is a matter of work in progress.


\bibliographystyle{IEEEtran}
\bibliography{IEEEabrv,biblio}

\end{document}